\documentclass{article}
\usepackage{url}
\usepackage{graphicx}
\usepackage{picinpar}
\usepackage{enumitem}
\usepackage{amssymb}
\usepackage{amsmath}
\usepackage{booktabs}
\usepackage{multirow}
\usepackage[table]{xcolor}
\usepackage{wrapfig}
\usepackage{caption}

\usepackage{tabularx}
\usepackage{array}
\usepackage{colortbl}
\newtheorem{proof}{Proof}
\usepackage{algorithm}
\usepackage{algorithmic}
\newtheorem{proposition}{Proposition}

\newcolumntype{Y}{>{\centering\arraybackslash}X}

\definecolor{morandiBlue}{RGB}{222,235,234}
\usepackage[preprint]{corl_2026} 

\title{Preference-Calibrated Human-in-the-Loop Reinforcement Learning for Robotic Manipulation}

\author{
  Zeyi Liu$^{1,2}$,
  Guangyao Liu$^{3,2}$,
  Yinuo Qu$^{2}$,
  Yuquan Xue$^{2}$,
  Bofang Jia$^{2}$,\\
  \textbf{Chunhua Yang$^{1}$,}
  \textbf{Weihua Gui$^{1}$,}
  \textbf{Keke Huang$^{1,\dagger}$,}
  \textbf{Ziwei Wang$^{2,\dagger}$}
  \\
  $^1$Central South University,
  $^2$Nanyang Technological University,
  $^3$Zhejiang University\\
  \texttt{liuzeyi@csu.edu.cn, huangkeke@csu.edu.cn, ziwei.wang@ntu.edu.sg}
}

\begin{document}

\setlength{\textfloatsep}{6pt plus 1pt minus 1pt}
\setlength{\floatsep}{6pt plus 1pt minus 1pt}

\maketitle
\begingroup
\renewcommand{\thefootnote}{$\dagger$}
\footnotetext{Corresponding author.}
\endgroup

\begin{abstract}
Human-in-the-loop reinforcement learning (HIL-RL) improves sample efficiency in real-robot manipulation through online human intervention. However, successful trajectories may include suboptimal actions that deviate from the desired task-execution path and force human intervention. Existing HIL-RL methods typically apply the consistent credit assignment principle to all transitions, uniformly propagating discounted terminal rewards through suboptimal segments, ignoring the actual contribution of each transition to task success. This overestimates Q-values for critic learning and indirectly misguides actor updates toward suboptimal behavior patterns.
To this end, we propose PACT, a \textbf{P}reference-calibrated \textbf{A}ctor-\textbf{C}ritic \textbf{T}raining framework that leverages the implicit preference signals induced by intervention to perform credit reassignment on identified suboptimal segments while directly guiding policy training for unbiased critic-actor learning. 
Specifically, we first design a progress model that learns from human demonstration and identifies suboptimal segments for credit correction.
Then, from the human action and resampled policy action at the intervention state, we build preference pairs to define a counterfactual advantage that penalizes Bellman targets of the identified suboptimal segment, enabling directional credit calibration. Moreover, we directly align the policy with human corrective actions in the bounded mean space, providing an additional signal beyond critic-guided updates.
Across five real-robot manipulation tasks, PACT improves the average success rate by 24.5\% and achieves 1.3$\times$ faster convergence, improving both RL sample efficiency and performance. Code is available at \href{https://anonymous.4open.science/r/HILRL-A1X-BC05}{\textit{Code Link}}.
\end{abstract}



\section{Introduction}
Deep reinforcement learning (DRL) has made substantial progress in robotic manipulation by enabling agents to acquire complex behaviors through trial-and-error interaction~\citep{deng2025survey,cui2021toward}. However, simulation-trained policies often suffer from domain-transfer failures~\citep{yang2023sim}, while real-world exploration is constrained by high interaction cost and safety risks~\citep{chen2024rlingua}. Human-in-the-loop RL (HIL-RL) addresses these challenges by incorporating real-time operator intervention, achieving effective policy learning under limited interaction budgets and substantially improving the sample efficiency of real-robot RL~\citep{luo2025precise, chen2025conrft, deng2026e2hil, luo2024serl}.

Despite recent progress, existing HIL-RL methods remain sample-inefficient and require substantial human intervention during training. We argue that this limitation partly stems from their neglect of the intrinsic heterogeneity of intervention-containing successful trajectories, where normal exploration, suboptimal segments, and human recovery coexist. In such trajectories, task success does not imply that all actions are beneficial, since some policy actions may deviate from the desired task-execution path and force human intervention. However, current methods typically apply the same credit assignment principle to all transitions, treating them as homogeneous learning samples. Under sparse terminal rewards, this uniformly propagates the final success signal through suboptimal segments, overestimating their actual contribution to task success. The resulting inflated Q-values can further misguide actor updates, since the policy is optimized to maximize the critic estimate, thereby reinforcing similar suboptimal behavior patterns and slowing convergence.

\begin{figure}[t]\centering
    \includegraphics[width=13cm]{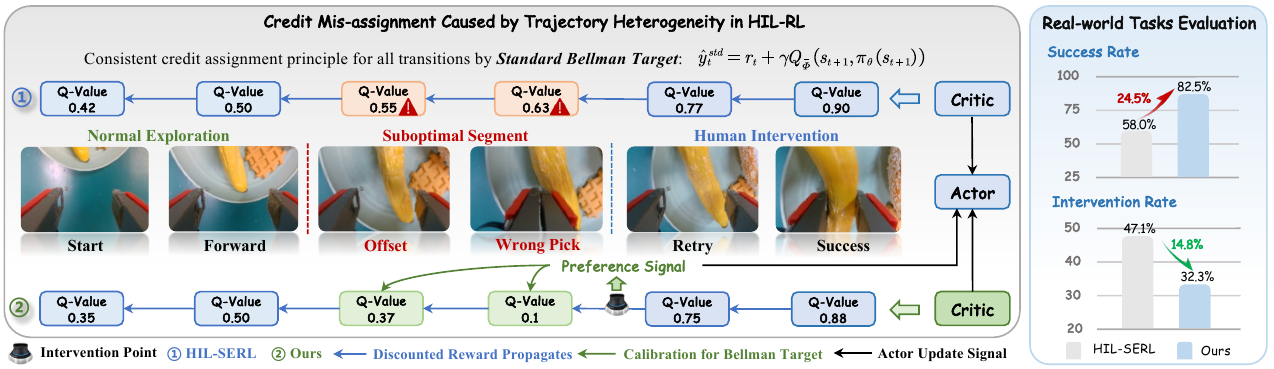}
    \caption{PACT calibrates credit assignment in intervention trajectories by localizing suboptimal segments and correcting inflated Q-values.}
    \label{framework}
\end{figure} 


To address this issue, we propose PACT, a \textbf{P}reference-calibrated \textbf{A}ctor-\textbf{C}ritic \textbf{T}raining framework for HIL-RL. Different from existing HIL-RL methods, PACT treats intervention-containing successful trajectories as heterogeneous learning samples, performs fine-grained credit calibration by identifying suboptimal action segments according to their contribution to task success, and guides policy and preference alignment. As shown in Figure~\ref{framework}, PACT constructs zero-cost preference pairs from human corrective actions and policy actions at the intervention state. These preferences are used to directionally penalize the Bellman targets of identified suboptimal segments, suppressing Q-value overestimation on the critic side, while directly guiding the actor toward human-preferred corrective actions beyond conventional critic-guided updates. Across five real-robot manipulation tasks, PACT improves the average success rate from 58.0\% to 82.5\%, reduces the intervention rate from 47.1\% to 32.3\%, and achieves 1.3$\times$ faster training than HIL-SERL, with no inference-time overhead. Our main contributions are summarized as follows:

\begin{itemize}[leftmargin=*]
    \item We propose PACT, a plug-and-play preference-calibrated actor-critic training framework that addresses credit misassignment in HIL-RL caused by trajectory heterogeneity.

    \item We design a progress model trained from human demonstrations to locate suboptimal segments for critic correction, and further build preference pairs to perform both critic and actor sides calibration, including credit reassignment and policy optimization guidance.

    \item We conduct extensive experiments on five real-robot manipulation tasks, demonstrating improved success rate, reduced human intervention, and faster convergence.
\end{itemize}


\section{Related Work}
\subsection{Human in the loop RL for Robotic Manipulation}

Real-robot RL learns policies through direct interaction with physical environments, but its sample efficiency is severely constrained by costly data collection and safety risks, especially compared with large-scale parallel simulation~\citep{zhang2021reinforcement, torne2024reconciling, ju2022transferring}. Prior work improves real-world efficiency through model-based RL~\citep{clavera2018model, lee2020guided}, hybrid approaches~\citep{zhang2021reinforcement, lei2025rl}, and priors from pretrained foundation models~\citep{zhong2024empowering, kawaharazuka2024real}, but these methods often introduce additional modeling, training, or deployment overhead. Human-in-the-loop RL offers a practical alternative by allowing operators to intervene during online exploration. Early human-robot collaborative learning builds on interactive imitation learning~\citep{kelly2019hg, ross2011reduction}, but purely imitative objectives cannot continuously improve through trial-and-error interaction. Recent methods therefore combine offline demonstrations with online RL~\citep{vecerik2017leveraging, ball2023efficient}. Among them, SERL~\citep{luo2024serl} provides a complete real-robot RL framework, and HIL-SERL~\citep{luo2025precise} further incorporates real-time intervention trajectories into replay-based training, achieving strong performance on dexterous manipulation tasks. Nevertheless, existing HIL-RL methods still optimize all transitions uniformly, overlooking the heterogeneity in intervention-containing successful trajectories and thus inducing credit misassignment. In this paper, we address this by exploiting intervention-induced preferences for credit reassignment on identified suboptimal segments.

\subsection{Q-value Correction and Preference Learning}

Q-value overestimation is a long-standing challenge in reinforcement learning. Existing methods mainly address it by improving critic estimation, such as Double DQN~\citep{van2016deep}, TD3~\citep{fujimoto2018addressing}, and CQL~\citep{kumar2020conservative}, or by reshaping Bellman learning with additional supervisory signals. The latter includes reward shaping with dense rewards~\citep{cao2024survey,memarian2021self,devidze2022exploration}, inverse RL from expert demonstrations~\citep{phan2023driveirl}, and preference-based learning from human comparisons~\citep{christiano2017deep, rafailov2023direct}. However, these methods typically rely on global trajectory-level correction, external reward models, or costly human annotation, and thus do not address segment-level credit errors inside intervention-containing trajectories. In contrast, HIL-RL naturally provides a zero-cost preference signal: at each intervention point, the human corrective action is preferred to the policy action. Based on this signal, our method performs preference-calibrated segment-level credit correction to reshape Bellman targets and further constrains actor learning for more sample-efficient policy optimization.
\section{Methodology}
In this section, we first introduce the problem of credit misassignment in HIL-RL (\ref{sec:3.1}). Then, we present a lightweight progress model used to localize suboptimal segments (\ref{sec:3.2}). Building on this, we propose a preference-aware counterfactual Q-value correction method to suppress overestimation (\ref{sec:3.3}). Finally, we introduce a preference auxiliary policy optimization that directly guides policy learning at intervention points (\ref{sec:3.4}). Our overall pipeline can be seen in Figure~\ref{method}.

\subsection{Preliminaries and Problem Statement} \label{sec:3.1}

Real-robot manipulation can be formulated as a Markov Decision Process $\mathcal{M} = (\mathcal{S}, \mathcal{A}, \mathcal{P}, r, \gamma)$, where $\mathcal{S}$ is the state space, $\mathcal{A}$ is the action space, $\mathcal{P}(s_{t+1}| s_t, a_t)$ is the transition dynamics, $r(s_t, a_t)$ is the reward function, and $\gamma \in [0, 1)$ is the discount factor. The objective is to learn a policy $\pi_\theta(a_t |s_t)$ that maximizes the expected cumulative discounted return $\mathbb{E}_{\pi_\theta}\left[\sum_{t=0}^{T} \gamma^t r(s_t, a_t)\right]$.
 
Recent HIL-RL methods improve the sample efficiency by incorporating human intervention and optimizing the policy with RLPD~\citep{ball2023efficient}. The critic is trained with the standard Bellman target:
\begin{equation}
\hat{y}_t^{std} = r_t + \gamma Q_{\bar{\phi}}(s_{t+1}, \pi_{\theta}(s_{t+1})), \quad \mathcal{L}_{\text{critic}} = \mathbb{E}_{(s_t,a_t,r_t,s_{t+1}) \sim \mathcal{B}} \left[ \left( Q_\phi(s_t, a_t) - \hat{y}_t^{std} \right)^2 \right],
\label{eq:bellman_std}
\end{equation}
where $\mathcal{B}$ is the replay buffer, and $Q_{\bar{\phi}}\triangleq \min_{i=1,2}Q_{\bar{\phi}_i}$ denotes the clipped target Q-value. The actor is updated by maximizing the critic estimate with entropy regularization:
\begin{equation}
\mathcal{L}_{\mathrm{actor}}
=
-\mathbb{E}_{s_t\sim\mathcal{B}}
\left[
Q_\phi(s_t,\pi_\theta(s_t))
+
\alpha\mathcal{H}(\pi_\theta(\cdot|s_t))
\right],
\label{eq:actor_std}
\end{equation}
where $\alpha$ is the temperature coefficient and $\mathcal{H}(\cdot)$ is the entropy loss.

However, this objective applies the same backup to all transitions. An intervention-containing successful trajectory is heterogeneous and can be written as
$\tau=\tau^{\mathrm{norm}}\circ\tau^{\mathrm{sub}}\circ\tau^{\mathrm{hum}}$, where
$\tau^{\mathrm{norm}}=\{s_t,a_t\}_{t=0}^{t_a-1}$,
$\tau^{\mathrm{sub}}=\{s_t,a_t\}_{t=t_a}^{t_h-1}$, and
$\tau^{\mathrm{hum}}=\{s_t,a_t\}_{t=t_h}^{T-1}$ denotes the normal, suboptimal, and human intervention segments, respectively. Under sparse terminal rewards, Eq.~(1) propagates the success signal through the entire trajectory, assigning undeserved credit to actions in $\tau^{\mathrm{sub}}$ and resulting in Q-values overestimation. In light of this, we aim to address this issue by representing such trajectories as structured heterogeneous data and handling them at a finer granularity during training, enabling more accurate credit assignment and more efficient policy optimization.

\begin{figure}[t]\centering
    \includegraphics[width=14cm]{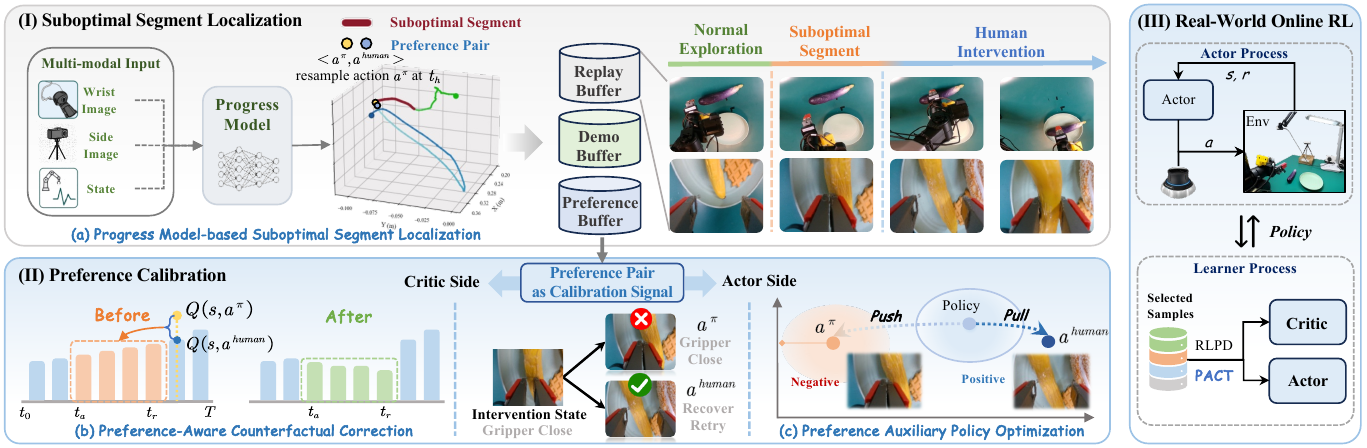}
    \caption{PACT converts intervention-containing trajectories into preference-calibrated critic and actor updates for sample-efficient real-world HIL-RL.}
    \label{method}
\end{figure} 

\subsection{Lightweight Progress Model-based Suboptimal Segment Localization} \label{sec:3.2}
To localize the suboptimal segments, we design a lightweight task progress estimator $f_\psi: \mathcal{O} \rightarrow [0, 1]$ that maps observations to a scalar progress value. It is trained in a fully self-supervised manner using only the demonstration data $\mathcal{D}_{\text{demo}}$, with \emph{zero} additional data collection or annotation cost.

\textbf{Input Representation.} We employ an encoder $g_{\text{img}}:  g_{\text{res}} \cdot h$ to encode image data and a learnable MLP $g_{\text{prop}}:\mathbb{R}^{7}\rightarrow \mathbb{R}^{d_p}$ to encode robot proprioceptive data, where $g_{res}:\mathbb{R}^{H \times W \times 3} \rightarrow \mathbb{R}^{d_v}$ is a frozen pretrained ResNet encoder, and $h: \mathbb{R}^{d_v} \rightarrow \mathbb{R}^{d_z}$ is a learnable projection. Let $I^{\text{w}}_t, I^{\text{s}}_t$ denote wrist and side camera images. The observation is constructed from residuals relative to the initial frame and proprioceptive displacement:
\begin{equation}
o_t = \left[  g_{\text{img}}(I^{\text{w}}_t - I^{\text{w}}_0),\  g_{\text{img}}(I^{\text{s}}_t -  I^{\text{s}}_0),\ g_{\text{prop}}(q_t - q_0) \right], \quad \hat{p}_t = f_\psi(o_t).
\label{eq:progress_input}
\end{equation}
The progress model $f_\psi$ then predicts a scalar value $\hat{p}_t = f_\psi(o_t) \in [0, 1]$.

 \textbf{Self-Supervised Training.} Building on the assumption that successful demonstration trajectories are near-optimal and exhibit smooth, monotonic task progress, we use uniform progress $p^*_t = t/T$ as pseudo-labels for a trajectory of $T$ steps. Furthermore, the regression loss with a discounted Monte Carlo return is introduced to strengthen the terminal supervision over the final segment $\mathcal{T}_{\text{end}}$:
\begin{equation}
\mathcal{L}_{\text{progress}} =\frac{1}{T} \sum_{t=0}^{T-1} \left( f_\psi(o_t) - \frac{t}{T} \right)^2 + \frac{\lambda_{\text{mc}}}{|\mathcal{T}_{\text{end}}|} \sum_{t \in \mathcal{T}_{\text{end}}} \left( f_\psi(o_t) - \sum_{k=0}^{T-t-1} \gamma^k r_{t+k} \right)^2.
\label{eq:loss_progress}
\end{equation}
\textbf{Suboptimal Segment Localization.} We perform progress inference on intervention-containing successful trajectories to localize suboptimal segments by localizing abnormal progress patterns before interventions. Let $\hat p_t=f_\psi(o_t)$ and $\Delta_w(t)=\hat p_{t+w}-\hat p_t$, we identify a regression anchor when the predicted progress within a short future window drops below the current progress by a margin:
\begin{equation}
t_a=\min\left\{t<t_h \mid 
\exists w\in\{1,\ldots,\min(W,t_h-t)\},\
\hat p_{t+w}<\hat p_t-\delta_{\mathrm{reg}}
\right\},
\end{equation}
where the recovery point is defined as the first timestep after $t_a$ with $K$ consecutive progress increases:
\begin{equation}
t_r=\min\left\{t>t_a \mid 
\hat p_{t+k}-\hat p_{t+k-1}>0,\ \forall k=1,\ldots,K
\right\}.
\end{equation}
If no recovery is found before intervention, we set $t_r=t_h-1$. The resulting interval $[t_a,t_r]$ is treated as a candidate suboptimal segment for subsequent critic correction.




\subsection{Preference-Aware Counterfactual Correction} \label{sec:3.3}

Given the identified suboptimal segments, PACT calibrates critic learning by converting preferences into target-level corrections. The key idea is to measure how much the current critic violates the preference that human corrective actions should be favored over suboptimal policy actions.

\textbf{Counterfactual Advantage}. Human intervention induces a local preference: at the intervention state, the corrective action $a^{\mathrm{human}}_{t_h}$ should be valued above the policy action $a^{\pi}_{t_h}$, i.e., $Q(s_{t_h},a^{\mathrm{human}}_{t_h}) > Q(s_{t_h},a^{\pi}_{t_h})$ under the desired value function. To operationalize this preference, we introduce a counterfactual advantage that measures the preference-violating value gap under the current target critic. This term provides a directional calibration signal for critic learning and determines how much over-credited value should be removed from the preceding suboptimal segment:
\begin{equation}
A_{\mathrm{cf}} =
\mathrm{sg}\left[
\max\left(
Q_{\bar{\phi}}(s_{t_h},a^\pi_{t_h})
-
Q_{\bar{\phi}}(s_{t_h},a^{\mathrm{human}}_{t_h}),
0
\right)
\right],
\label{eq:advantage}
\end{equation}
where $\mathrm{sg}[\cdot]$ denotes the stop-gradient operator, and the non-negative clipping ensures directional suppression of preference-violating critic estimates.

\textbf{Corrected Bellman Target.} Although $A_{\mathrm{cf}}$ is computed from a single intervention point, the detected interval $[t_a,t_r]$ corresponds to the policy-executed segment leading to the intervention. Therefore, we use $A_{\mathrm{cf}}$ as a segment-level correction signal to calibrate its Bellman targets.

Considering early actions near $t_a$ may still reflect exploratory behavior and have weaker causal evidence for intervention, whereas later actions are more strongly associated with the detected suboptimal execution. We introduce a position-aware exponential weight $\omega(t) = \exp\left( -\beta \cdot \frac{t_r - t}{t_r - t_a} \right),$
where $\beta$ controls how sharply the correction increases along the segment. The exponential form provides conservative correction for early uncertain actions while concentrating stronger calibration on later high-confidence suboptimal actions.

Let $\hat{y}^{\mathrm{std}}_t = r_t + \gamma Q_{\bar{\phi}}(s_{t+1}, \pi_\theta(s_{t+1}))$ denote the standard Bellman target. PACT penalize the Bellman target by subtracting the position-weighted counterfactual correction:
\begin{equation}
\hat{y}^{\mathrm{corr}}_t =
\begin{cases}
\hat{y}^{\mathrm{std}}_t - \omega(t) \cdot A_{\mathrm{cf}}, & \text{if } t \in [t_a, t_r], \\
\hat{y}^{\mathrm{std}}_t, & \text{otherwise}.
\end{cases}
\label{eq:corrected_bellman}
\end{equation}
The critic is then trained by minimizing the temporal-difference error to the corrected target:
\begin{equation}
\mathcal{L}_{\mathrm{critic}}
=
\frac{1}{2}
\sum_{i=1}^{2}
\mathbb{E}_{(s_t,a_t,r_t,s_{t+1})\sim\mathcal{B}}
\left[
\left(
Q_{\phi_i}(s_t,a_t)-\hat y_t^{\mathrm{corr}}
\right)^2
\right].
\label{eq:critic_loss}
\end{equation}

\subsection{Preference Auxiliary Policy Optimization} \label{sec:3.4}

While the corrected critic suppresses value inflation on suboptimal segments, the intervention signal can also directly guide the policy distribution learning. For each intervention state $s_{t_h}$, we obtain a zero-cost preference pair $(a^{\mathrm{human}}, a^\pi)$. A direct log-probability preference objective, as in discrete-action preference optimization~\cite{hong2024orpo}, is ill-conditioned for continuous tanh-squashed Gaussian actors. As training converges and the policy variance shrinks, log-probability gaps become dominated by the $1/\sigma_d^2$ term, leading to unstable gradients. We therefore impose the preference in the bounded mean-action space and backpropagate only through the squashed mean $\bar a_\theta(s)=\tanh(\mu_\theta(s))$:
\vspace{-0.3em}
\begin{equation}
\mathcal{L}_{\text{pref}} = E_{(s, a^{\text{human}}, a^{\pi}) \sim \mathcal{B}_{\text{pref}}} \left[ \left\| \tanh(\mu_\theta(s)) - a^{\text{human}} \right\|^2 + \frac{1}{\left\| \tanh(\mu_\theta(s)) - a^{\pi} \right\|^2+\epsilon} \right],
\label{eq:pref_loss}
\end{equation}
where the first term aligns the actor with the corrective action, the second repels it from the intervention-triggering policy action, and $\epsilon$ prevents numerical singularity. The final actor objective combines the standard RLPD actor loss with the preference-calibration term:
\begin{equation}
\mathcal{L}_{\text{actor}} = \mathcal{L}_{\text{RLPD}} + \lambda_{\text{pref}} \mathcal{L}_{\text{pref}},
\label{eq:actor_total}
\end{equation}
where $\lambda_{\mathrm{pref}}$ controls the strength of actor-side preference calibration.


\section{Experimental Results}

We evaluate PACT on real-robot manipulation tasks through four questions. \textbf{(1)}. Does PACT improve sample efficiency and performance (Section \ref{subsection-4.2})? \textbf{(2)}. Can the progress model reliably identify suboptimal segments before human interventions (Section \ref{subsection-4.3})? \textbf{(3)}. Can preference-aware counterfactual correction mitigate Q-value overestimation during training (Section \ref{subsection-4.4})? \textbf{(4)}. How does each component contribute to the policy performance improvement (Section \ref{subsection-4.5})?

\subsection{Experiment Setup} \label{subsection-4.1}

\textbf{Real Robot Task.} We evaluate PACT on five real-world manipulation tasks with increasing difficulty: \textit{Press}, \textit{Insertion}, \textit{Pick}, \textit{Pick \& Place}, and \textit{Assembly}. For each task, we collect 20 demonstration trajectories and store them in $\mathcal{D}_{demo}$ for progress model training and early-stage RL stabilization. All experiments are conducted on a \textit{Galaxea A1X} 6-DoF arm.

\textbf{Baseline and Evaluation Metrics.}
We adopt HIL-SERL~\cite{luo2025precise} as the primary baseline. It incorporates human intervention trajectories into replay-based off-policy training with the entropy-regularized RLPD objective and represents the state-of-the-art HIL-RL framework for real-world manipulation. We evaluate each method using three metrics: the final success rate (SR), the training time, and the average human intervention rate (IR). The human intervention rate is defined as the ratio between the number of intervention steps and the total robot interaction steps.

\subsection{Comparisons with the State-of-the-Art}\label{subsection-4.2}

Table~\ref{comparison_exp} reports final quantitative results for PACT and HIL-SERL, while Fig.~\ref{qualitative_analysis} shows their training dynamics in intervention and success rates across five tasks. PACT consistently improves final performance and sample efficiency, increasing the average success rate from $58.0\%$ to $82.5\%$ and reducing the average intervention rate from $47.1\%$ to $32.3\%$. The improvement is especially pronounced on \textit{Assembly}, the most complex task, where PACT reaches a $62.5\%$ success rate in 122 minutes, compared with only $10.0\%$ for HIL-SERL over 140 minutes. PACT also shortens average training time from $80.7$ to $63.0$ minutes, indicating faster convergence under the same real-robot training pipeline. The learning curves in Fig.~\ref{qualitative_analysis} further show that PACT reduces human intervention more rapidly while achieving faster success-rate growth, suggesting that preference-aware segment-level correction improves policy learning rather than merely relying on additional human recovery.
\begin{table}[t]
\centering
\tiny
\renewcommand{\arraystretch}{0.9}
\setlength{\tabcolsep}{4pt}
\newcommand{\hl}[1]{\cellcolor{morandiBlue}{#1}}
\caption{Real-world Results. Comparison of the success rate, intervention rate, and training time between PACT and HIL-SERL
across five real-world tasks}
\resizebox{0.9\linewidth}{!}{
\begin{tabular}{llcccccc}
\toprule
\textbf{Metric} & \textbf{Method} 
& \textbf{Press} 
& \textbf{Insertion} 
& \textbf{Pick} 
& \textbf{Pick \& Place} 
& \textbf{Assembly}
& \textbf{Average} \\
\midrule

\multirow{2}{*}{\textbf{Success Rate (\%) \textcolor{green!60!black}{$\uparrow$}}}
& HIL-SERL & 85.0 & 73.0 & 67.0 & 55.0 & 10.0 & 58.0 \\
& \hl{\textbf{PACT}} 
& \hl{\textbf{95.0}} 
& \hl{\textbf{90.0}} 
& \hl{\textbf{95.0}} 
& \hl{\textbf{70.0}} 
& \hl{\textbf{62.5}} 
& \hl{\textbf{82.5}} \\
\midrule

\multirow{2}{*}{\textbf{Intervention Rate (\%) \textcolor{red}{$\downarrow$}}}
& HIL-SERL & 43.6 & 45.3 & 35.8 & 54.9 & 56.0 & 47.1 \\
& \hl{\textbf{PACT}} 
& \hl{\textbf{19.1}} 
& \hl{\textbf{32.5}} 
& \hl{\textbf{26.1}} 
& \hl{\textbf{36.6}} 
& \hl{\textbf{47.2}} 
& \hl{\textbf{32.3}} \\
\midrule

\multirow{2}{*}{\textbf{Run Time (Minutes) \textcolor{red}{$\downarrow$}}}
& HIL-SERL & 23.2 & 51.8 & 75.0 & 110.8 & 142.6 & 80.7 \\
& \hl{\textbf{PACT}} 
& \hl{\textbf{21.4}} 
& \hl{\textbf{30.3}} 
& \hl{\textbf{46.6}} 
& \hl{\textbf{94.5}} 
& \hl{\textbf{122.1}} 
& \hl{\textbf{63.0}} \\
\bottomrule
\end{tabular}
}
\label{comparison_exp}
\vspace{-0.8em}
\end{table}

\begin{figure}[!t]\centering
    \includegraphics[width=13cm]{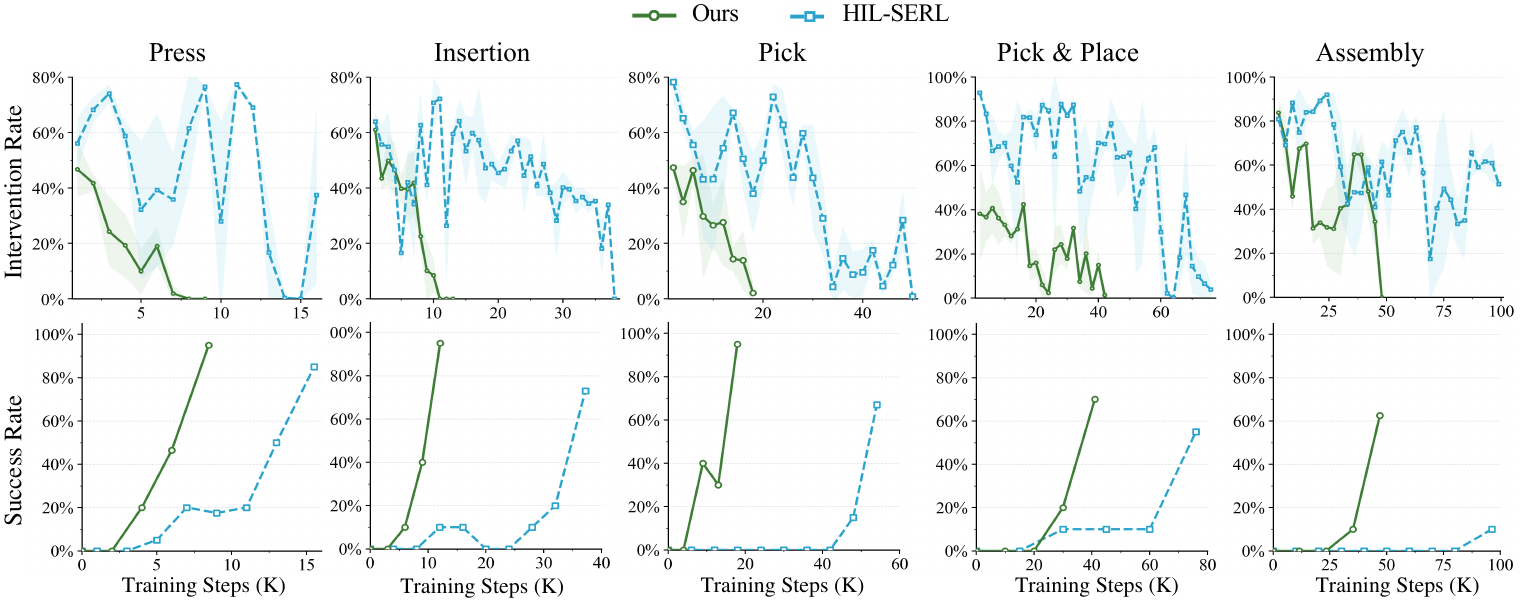}
    \caption{Real-robot training dynamics of PACT and HIL-SERL across five manipulation tasks.}
    \label{qualitative_analysis}
\end{figure} 

\vspace{-0.3em}
\subsection{Progress Model Analysis} \label{subsection-4.3}
\textbf{Qualitative Results.}
We first examine whether the progress model localizes behaviorally meaningful suboptimal segments for subsequent credit correction. Fig.~\ref{Progress Model Analysis} visualizes representative real-robot rollouts, where the end-effector trajectories are plotted in 3D space and the segments detected by the progress model are highlighted in red. These segments align with typical failure patterns, such as hesitation in \textit{Insertion}, wrong grasp in \textit{Pick}, and directional deviation in \textit{Pick \& Place}, suggesting that the identified results are closely associated with degraded execution, providing a basis for targeting segment-level credit correction at the detected suboptimal parts of intervention-containing successful trajectories.
\begin{figure}[h]\centering
    \includegraphics[width=13cm]{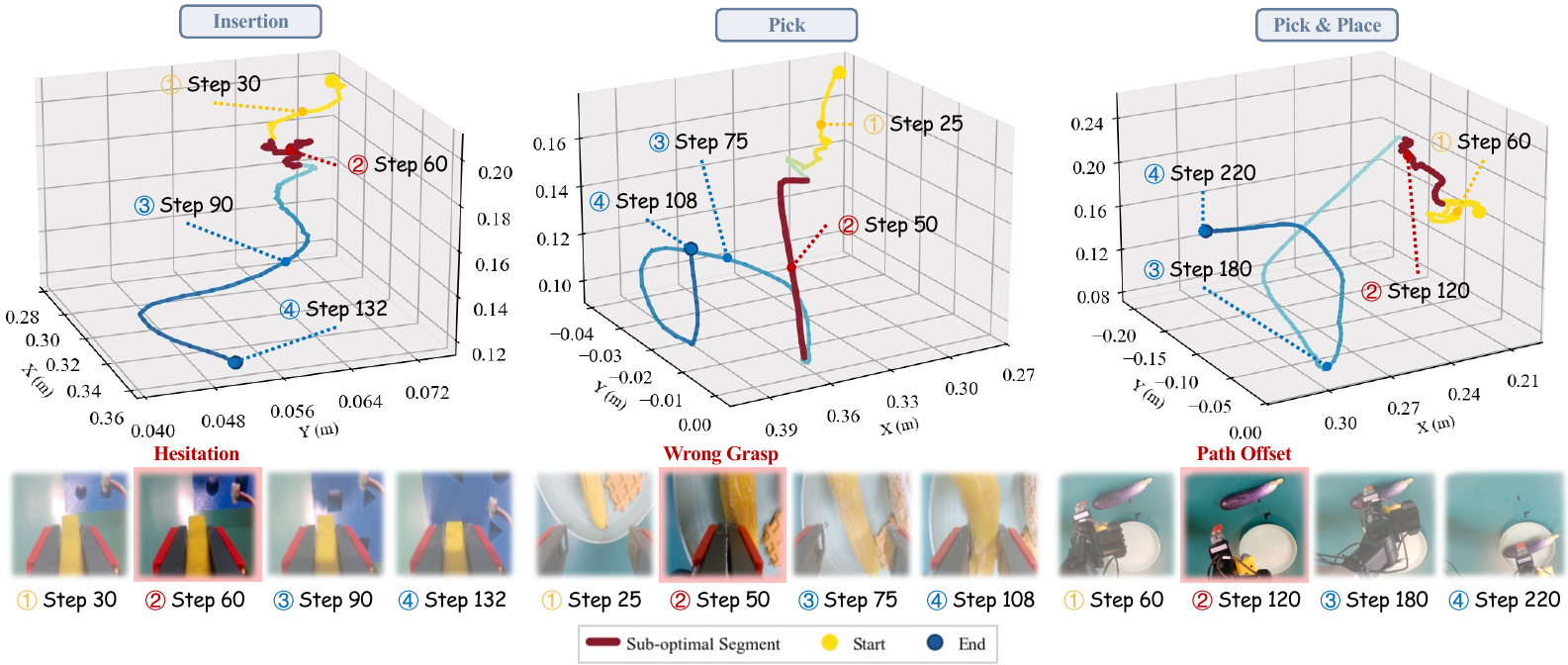}
    \caption{Progress-based suboptimal segment localization. The red segments denote identified suboptimal portions, which correspond to wrong behavior in RL rollouts.}
    \label{Progress Model Analysis}
\end{figure} 

\textbf{Quantitative Evaluation.}
Since the progress model is trained without external supervised labels, we assess its localization results using two post-hoc, model-independent kinematic metrics: path efficiency $\eta$ and direction deviation $\Delta\theta$. For a trajectory segment $[t_{\mathrm{start}}, t_{\mathrm{end}}]$, they are defined as:
\begin{equation}
\eta =
\frac{
\|\mathrm{pos}(t_{\mathrm{end}})-\mathrm{pos}(t_{\mathrm{start}})\|
}{
\sum_{t=t_{\mathrm{start}}}^{t_{\mathrm{end}}-1}
\|\mathrm{pos}(t+1)-\mathrm{pos}(t)\|
},
\quad
\Delta\theta =
\arccos
\left(
\frac{
\vec{d}_{\mathrm{seg}} \cdot \vec{d}_{\mathrm{demo}}
}{
\|\vec{d}_{\mathrm{seg}}\| \|\vec{d}_{\mathrm{demo}}\|
}
\right),
\label{eq:quality_metrics}
\end{equation}
where $\vec{d}_{\mathrm{seg}}=\mathrm{pos}(t_{\mathrm{end}})-\mathrm{pos}(t_{\mathrm{start}})$, $\vec{d}_{\mathrm{demo}}=\mathrm{pos}_{\mathrm{demo}}(t_{\mathrm{match}}+N)-\mathrm{pos}_{\mathrm{demo}}(t_{\mathrm{match}})$, and $N=t_{\mathrm{end}}-t_{\mathrm{start}}$. Here, $t_{\mathrm{match}}$ is the timestep in $\mathcal{D}_{\mathrm{demo}}$ whose end-effector position is nearest to $\mathrm{pos}(t_{\mathrm{start}})$. These metrics capture complementary and task-agnostic aspects of motion quality, where $\eta$ measures motion redundancy by comparing net displacement with accumulated path length and $\Delta\theta$ measures execution drift as the angular mismatch between the segment direction and the nearest matched demonstration direction.
As shown in Table~\ref{tab:progress_quality}, identified suboptimal segments consistently exhibit lower path efficiency and larger direction deviation, confirming that the progress model can effectively identify low-quality segments for subsequent Q-value correction.
\begin{table}[h]
\centering
\caption{
Trajectory quality comparison between normal and identified suboptimal segments.
$\eta$ denotes path efficiency, and $\Delta\theta$ denotes the direction deviation to the nearest demonstration segment.
}
\small
\label{tab:progress_quality}
\renewcommand{\arraystretch}{1.1}
\resizebox{0.9\linewidth}{!}{
\begin{tabular}{llccccc}
\toprule
\textbf{Metric} & \textbf{Segment}
& \textbf{Press}
& \textbf{Insertion}
& \textbf{Pick}
& \textbf{Pick \& Place}
& \textbf{Assembly} \\
\midrule

\multirow{2}{*}{\textbf{Path Efficiency $\eta$ \textcolor{green!60!black}{$\uparrow$}}}
& Normal
& .530 & .577 & .551 & .537 & .495 \\
& \cellcolor{cyan!8}Suboptimal
& \cellcolor{cyan!8}.521
& \cellcolor{cyan!8}.471
& \cellcolor{cyan!8}.490
& \cellcolor{cyan!8}.432
& \cellcolor{cyan!8}.473 \\

\midrule

\multirow{2}{*}{\textbf{Direction Deviation $\Delta\theta$ \textcolor{red}{$\downarrow$}}}
& Normal
& 47.1$^\circ$ & 41.7$^\circ$ & 40.4$^\circ$ & 54.8$^\circ$ & 38.8$^\circ$ \\
& \cellcolor{cyan!8}Suboptimal
& \cellcolor{cyan!8}67.4$^\circ$
& \cellcolor{cyan!8}67.4$^\circ$
& \cellcolor{cyan!8}68.3$^\circ$
& \cellcolor{cyan!8}83.1$^\circ$
& \cellcolor{cyan!8}63.7$^\circ$ \\

\bottomrule
\end{tabular}
}
\end{table}

\subsection{Q-value Correction Analysis}
\label{subsection-4.4}

We further evaluate whether preference-aware counterfactual correction suppresses Q-value overestimation during training. Since true action values are unavailable in real-robot experiments, we use the Monte Carlo return $G_t$ only as a diagnostic reference and compare the learned critics of HIL-SERL and PACT on intervention-containing successful episodes.

Fig.~\ref{critic}(a) compares the target Q-value evolution during training on the \textit{Insertion} task. HIL-SERL produces rapidly increasing target values, suggesting that uniform TD backup propagates terminal rewards to suboptimal pre-intervention segments and leads to value inflation. In contrast, PACT keeps the value estimates more conservative through preference-aware correction. Fig.~\ref{critic}(b) shows a representative case. Within the detected correction segment, HIL-SERL assigns values higher than the MC reference, whereas PACT substantially lowers the estimated values. Fig.~\ref{critic}(c) further aggregates the normalized critic bias $Q(s,a)-G_t$ over 20 episodes, where HIL-SERL exhibits consistently positive bias while PACT largely suppresses this overestimation. Table~\ref{tab:q_bias} confirms this trend across all five tasks.

\begin{figure}[h]\centering
    \includegraphics[width=13.5cm]{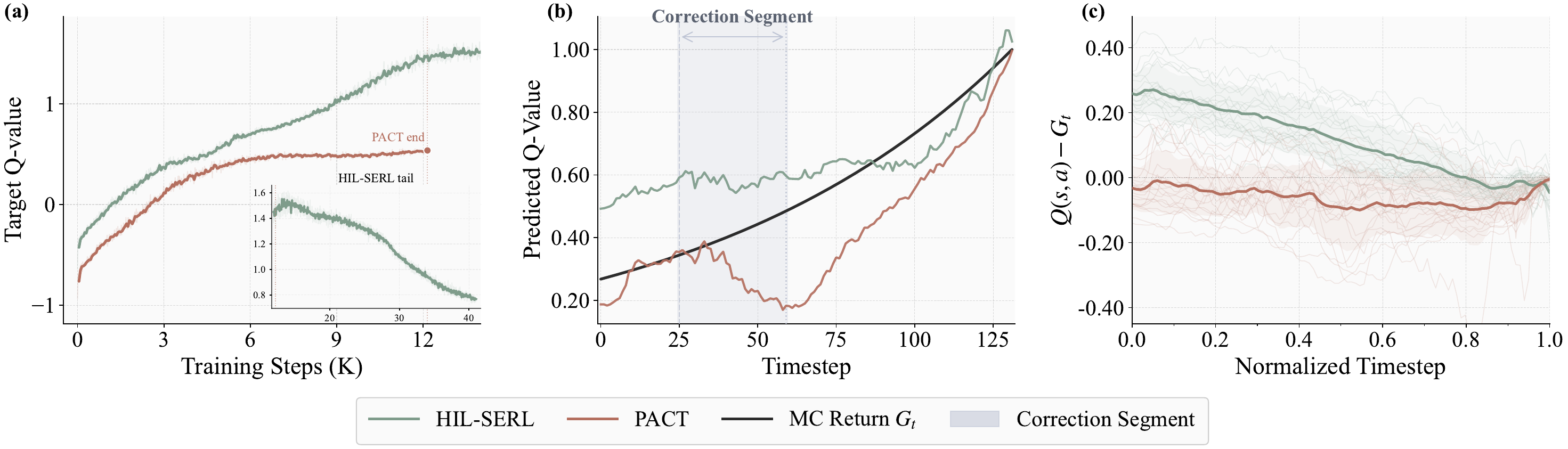}
    \caption{Q-value correction analysis. 
(a) Training stage Q-value evolution. (b) Single-trajectory value profile. (c) Normalized critic bias over intervention-containing successful episodes.}
    \label{critic}
\end{figure}

\begin{table}[t]
\centering
\caption{
Average critic bias $Q(s,a)-G_t$ on detected suboptimal segments.
Positive values indicate overestimation relative to the Monte Carlo return.
}
\label{tab:q_bias}
\small
\setlength{\tabcolsep}{3pt}
\renewcommand{\arraystretch}{1.1}
\resizebox{0.7\linewidth}{!}{
\begin{tabular}{lccccc}
\toprule
\textbf{Method} & \textbf{Press} & \textbf{Insertion} & \textbf{Pick}  & \textbf{Pick \& Place} & \textbf{Assembly} \\
\midrule
HIL-SERL & $+0.029$ & $+0.114$ & $+0.019$  & $+0.054$ & $+0.060$ \\
PACT     & $-0.042$ & $-0.061$ & $-0.107$  & $-0.146$ & $-0.152$ \\
\bottomrule
\end{tabular}}
\end{table}

\vspace{-0.8em}
\subsection{Ablation Study}
\label{subsection-4.5}

\begin{wraptable}{r}{0.42\linewidth}
\vspace{-1.0em}
\centering
\caption{Ablation study on \textit{Insertion}.}
\label{tab:ablation}
\renewcommand{\arraystretch}{1.12}
\resizebox{\linewidth}{!}{
\begin{tabular}{lcc}
\toprule
\textbf{Setting} & \textbf{SR@30min} $\uparrow$ & \textbf{IR@30min} $\downarrow$ \\
\midrule
w/o. C1 & 60\% & 41.3\% \\
w/o. C2 & 80\% & 45.0\% \\
\midrule
\textbf{PACT} & \textbf{90\%} & \textbf{32.5\%} \\
\bottomrule
\end{tabular}
}
\vspace{-1.0em}
\end{wraptable}

We conduct ablations on the \textit{Insertion} task. 
Here, C1 denotes the critic-side correction module, including progress-based segment localization with Q-value correction, and C2 denotes the actor-side preference loss. 
As shown in Table~\ref{tab:ablation}, removing C1 causes the largest success-rate drop, while removing C2 mainly increases the intervention rate. 
The full PACT achieves the best result, suggesting that critic-side correction and actor-side preference learning are complementary.

\section{Conclusion}
In this paper, we propose a Preference-Calibrated HIL-RL method that treats intervention trajectories as structured heterogeneous data. By localizing suboptimal segments, correcting inflated Bellman targets with intervention-induced preferences, and aligning the actor toward human preference actions, PACT improves success rates, reduces human intervention, and accelerates training across five real-robot manipulation tasks. Further analyses show that the progress model, critic correction, and actor-side preference objective contribute complementary benefits for sample-efficient real-world robot learning.

\section{Limitation}
While PACT substantially improves sample efficiency in real-world HIL-RL, it does not fully resolve credit assignment under all manipulation conditions. To keep the method lightweight and annotation-free for real-robot training, the progress model relies on limited demonstration-based self-supervision and may be affected by non-monotonic trajectories or repetitive execution patterns. Since the subsequent critic correction depends on the localized suboptimal segments, such localization errors may further propagate to critic learning. Moreover, the proposed counterfactual advantage provides a directional preference-based correction rather than a precise per-step value target, making it more suitable for suppressing systematic overestimation than for fine-grained value calibration. Future work could incorporate semantic-rich progress models, VLM-assisted progress reasoning, and pretrained progress representations to improve robustness and enable more accurate credit reassignment in complex long-horizon real-robot tasks.

\acknowledgments{This work was supported by  the Singapore National Robotics Programme Research Project DS-RFM M25N4N2009, the National Natural Science Foundation of China 92467107, and the Scientific Research Innovation Capability Support Project for Young Faculty ZYGXQNJSKYCXNLZCXM-I27.}


\bibliography{main}  

\newpage

\appendix
\section*{Appendix}

\section{Video Demo}
A short video presents PACT, the preference-calibrated actor-critic training framework for real-robot manipulation. We first introduce how PACT localizes suboptimal segments and uses intervention-induced preferences to calibrate actor-critic training. Then we show results on five real-robot tasks, with success rate, intervention rate, and training time reported for each task. Finally, we visualize the progress model and critic model to further verify the effectiveness of PACT. Enjoy the demo!

\section{Task Description}
We evaluate PACT on five real-robot manipulation tasks as shown in Fig.~\ref{task} with different levels of difficulty, covering contact-rich interaction, object grasping, object relocation, and multi-stage manipulation. 

\begin{figure}[h]\centering
    \includegraphics[width=11cm]{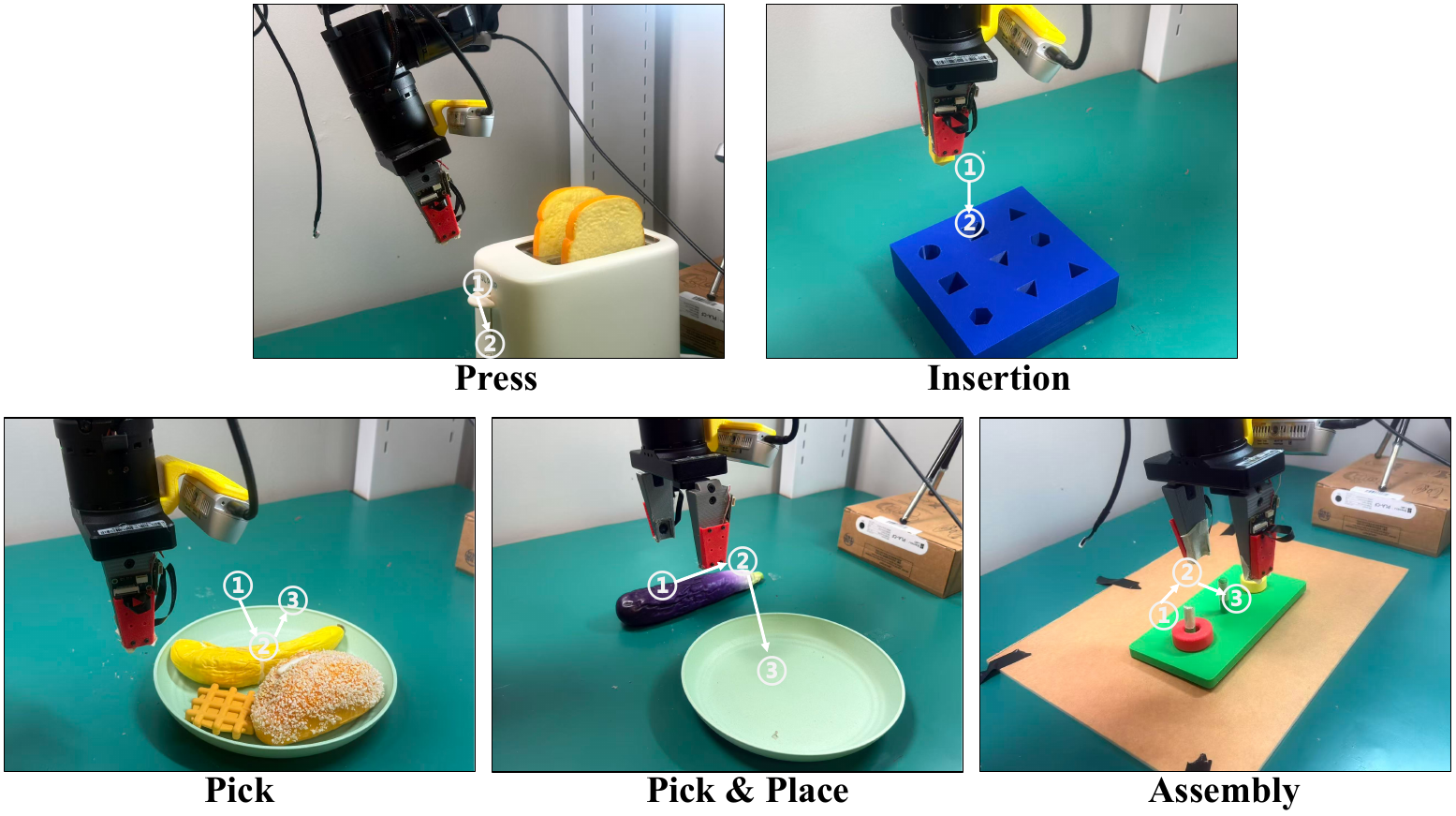}
    \caption{Real-world manipulation tasks used in our experiments. From left to right and top to bottom: Press, Insertion, Pick, Pick \& Place, and Assembly.}
    \label{task}
\end{figure} 

\paragraph{Press (Easy).}
A toaster is placed in the workspace. The robot must move its end effector to the target handle and press it down to trigger the mechanism. Success is achieved when the handle is fully pressed without missing the target or causing invalid contact.

\paragraph{Insertion.}
A shape-sorting board with several target holes is placed on the tabletop. The robot must align the block with the corresponding slot and insert it into the board. Success requires the block to be correctly inserted into the target hole without being stuck or dropped.

\paragraph{Pick.}
Several food-shaped objects are placed on a plate, and the robot is required to identify and grasp the banana. This task evaluates object-specific grasping in the presence of distractors. Success is achieved when the banana is securely grasped and lifted from the plate without disturbing other objects excessively.

\textbf{Pick \& Place.} An eggplant and a target plate are placed in the workspace. The robot must grasp the eggplant, transport it to the plate, and release it inside the target region. Success requires the object to be stably placed on the plate after release.

\textbf{Assembly (Hard).} A fixture board with multiple vertical pegs is placed in the workspace, and a red component is initially mounted on the third peg. The robot must approach the red component, lift it upward from the third peg, move it to the middle second peg, and insert it onto the target peg. Success is achieved when the red component is correctly transferred from the third peg to the second peg and remains stably mounted after execution.

\section{Implementation Details}
This section provides implementation details of the baseline methods and PACT. Unless otherwise specified, PACT follows the same actor-learner training pipeline, network backbone, replay sampling strategy, and optimization hyperparameters as the HIL-SERL baseline. The main differences lie in the progress-based segment localization, preference-based critic correction, and actor-side preference learning.
\subsection{Real-Robot Setup}
All experiments are conducted on a real tabletop manipulation platform, as shown in Fig.~\ref{robot}. The platform consists of a 6-DoF Galaxea A1X robotic arm equipped with a wrist-mounted Intel RealSense D435i camera. We additionally use a side-view Intel RealSense D435i camera to provide an external observation of the workspace. Human operators monitor the robot execution and provide online interventions through a 3D space mouse.
\begin{figure}[h]\centering
    \includegraphics[width=10cm]{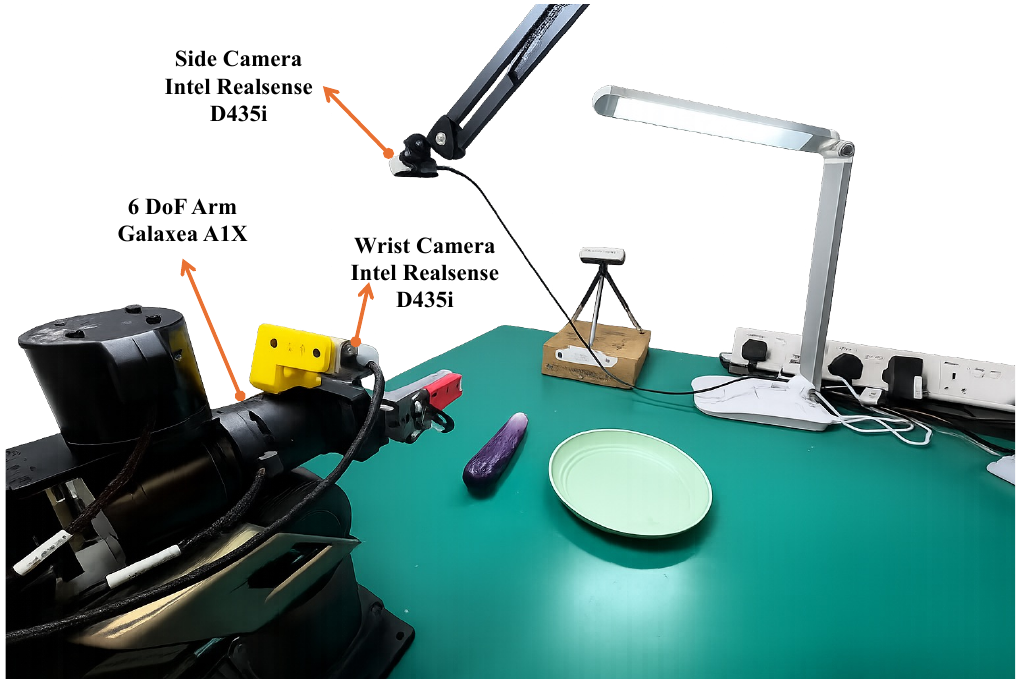}
    \caption{Real-world robot setup. We employ the Galaxea A1X as our real-world reinforcement learning platform.}
    \label{robot}
\end{figure}

\subsection{Training Settings of Baseline Methods}

\paragraph{Training framework.}
We build our implementation on the asynchronous actor-learner training pipeline used in HIL-SERL. The actor interacts with the real robot and stores online transitions into the replay buffer, while the learner samples mini-batches from both the online replay buffer and the demonstration buffer to update the networks. The updated learner parameters are periodically synchronized to the actor.

\paragraph{Observation and action space.}
For all real-robot tasks, the policy receives two visual observations and one proprioceptive state. The visual observations include a wrist-camera image and a side-camera image. The proprioceptive input contains the end-effector pose and gripper state. The environment action has seven dimensions, where the first six dimensions correspond to continuous end-effector control, and the last dimension corresponds to the gripper command. In the hybrid single-arm setting, the continuous actor outputs only the six-dimensional end-effector action, while a separate grasp critic handles the gripper.

\paragraph{Network architecture.}
The visual encoder uses a frozen pretrained ResNet-10 backbone with learned spatial pooling. Each camera image is encoded separately, and the two visual features are concatenated. The proprioceptive state is projected into a 64-dimensional feature and then concatenated with the visual feature to form the final observation latent. The actor is a tanh-squashed Gaussian policy with a two-layer MLP of hidden sizes $[256, 256]$. The continuous critic is a double-Q ensemble with the same hidden sizes. Both actor and critic use tanh activations and LayerNorm. In learned-gripper tasks, including \textit{Pick}, \textit{Pick \& Place}, and \textit{Assembly}, a separate grasp critic predicts scores for three discrete gripper actions. The SAC temperature is initialized to $10^{-2}$ and learned during training.

\paragraph{Demonstration and replay sampling.}
Demonstrations are not used to initialize the online replay buffer. Instead, they are loaded into a separate demonstration buffer before training. The online replay buffer starts empty and is filled by real-robot interaction. During learner updates, each mini-batch is formed by sampling half of the transitions from the online replay buffer and half from the demonstration buffer. In our experiments, each task uses 20 human demonstrations.

\subsection{Implementation Details of PACT}
\label{app}

PACT keeps the baseline actor-critic architecture and training protocol unchanged, and only modifies the training-time data annotation and loss computation. Specifically, PACT adds a progress model for suboptimal segment localization, a preference buffer for storing intervention-induced preference pairs, a critic-side target correction term, and an actor-side preference loss.

\paragraph{Progress Model Training.}
The progress model is trained offline from the 20 demonstrations before online RL.
It uses frozen ResNet-18 backbones producing 512-dim features each, and a 64-dim linear encoder for the 7-dim state difference $q_t - q_0$.
The three embeddings are concatenated and fed into a two-layer MLP head with hidden dimensions of 128 and 64, each followed by LayerNorm and ReLU, and a final Sigmoid activation to produce a scalar progress estimate $\hat{p}_t \in [0,1]$.
 
Training uses two self-supervised objectives with a combined loss $\mathcal{L} = \lambda_\text{prog}\mathcal{L}_\text{prog} + \lambda_\text{mc}\mathcal{L}_\text{mc}$, where $\mathcal{L}_\text{prog}$ regresses against the linear pseudo-label $p_t = t/T$, $\mathcal{L}_\text{mc}$ regresses against Monte Carlo returns computed with $\gamma=0.98$ under a sparse terminal reward, and is applied only over the terminal segment $\mathcal{T}_\text{end}$ to strengthen supervision near task completion. Default weights are $\lambda_\text{prog}=1.0$, $\lambda_\text{mc}=0.1$.

\paragraph{Suboptimal Segment Localization.}
At the end of each intervention-containing episode, the actor runs the progress model over the full trajectory and detects anomalies in the predicted progress curve.
A regression anomaly is flagged when the predicted progress drops by more than $\delta_\text{reg}=0.045$ within a future window of $W_{reg}=4$ steps.
Recovery is declared after $K=3$ consecutive progress increases.
For each intervention point $t_h$, the nearest preceding anomaly start is taken as the suboptimal segment anchor $t_a$, and the segment right boundary $t_r$ is set to the first timestep at which the predicted progress resumes monotonic increase. If no recovery is found before $t_h$, we set $t_r = t_h - 1$.
The resulting segment $[t_a, t_r]$ is annotated into the replay buffer via per-transition correction weight $\omega(t)$.

\paragraph{Preference Pair Construction and Calibration.}
At the first step of each human intervention, the actor records the current observation $s_{t_h}$, the human corrective action $a^h$, and the policy action $a^\pi$ resampled at the handover point.
These three fields are written as a preference pair into a dedicated preference buffer.
Only the first step of a continuous takeover is stored, to prevent a single intervention event from generating redundant pairs.

Each replay transition in $[t_a, t_r]$ is matched to its corresponding intervention state via a \texttt{segment\_id} field stored in the replay buffer, from which $A_\text{cf}$ is retrieved and applied with position-aware decay rate $\beta = 3.0$.
For the actor-side preference loss, the preference objective is imposed directly on the squashed mean action $\bar{a}_\theta(s) = \tanh(\mu_\theta(s))$ rather than on log-probabilities, so that the loss remains well-conditioned as policy variance shrinks during training.
The preference batch size is set to 16 in our implementation.

\paragraph{Hyperparameters.}
Table~\ref{tab:hyperparameters} summarizes all key hyperparameters.
 
\begin{table}[h]
\centering
\caption{Hyperparameters used in all experiments.}
\label{tab:hyperparameters}
\small
\begin{tabular}{lc}
\toprule
\textbf{Hyperparameter} & \textbf{Value} \\
\midrule
\multicolumn{2}{l}{\textbf{\textit{Shared between PACT and HIL-SERL}}} \\
Discount factor $\gamma$ & 0.98 \\
Target network update rate $\tau$ & 0.005 \\
Batch size & 256 (128 replay + 128 demo) \\
Actor / Critic / Progress model learning rate & $3\times10^{-4}$ \\
SAC temperature initialization & $10^{-2}$ \\
Critic ensemble size & 2 \\
Actor / Critic hidden size & 256 $\times$ 2 layers \\
Proprioceptive encoder output dim & 64 \\
Image resolution & $256\times256$ \\
Demonstrations per task & 20 \\
Episode horizon & 400 steps \\
Control frequency & 500 Hz \\
Initial random transitions before training & 100 transitions \\
Critic-to-actor update ratio & 2 \\
Parameter sync interval & 50 learner steps \\
Max training steps & $1\times10^{6}$ \\
\midrule
\multicolumn{2}{l}{\textbf{\textit{PACT-specific}}} \\
Progress model loss weights $(\lambda_\text{prog}, \lambda_\text{mc})$ & $(1.0,\ 0.1)$ \\
Terminal segment length $\|\mathcal{T}_\text{end}\|$ 
&
10\% of the episode\\
Regression window $W_{\mathrm{reg}}$ & 4 steps \\
Regression threshold $\delta_\text{reg}$ & 0.045 \\
Recovery length $K$ & 3 steps \\
Fallback window $W_{\mathrm{fb}}$ & 5 steps \\
Correction decay rate $\beta$ & 3.0 \\
Preference batch size & 16 \\
Preference loss weight $\lambda_\text{pref}$ & 0.2 \\
\bottomrule
\end{tabular}
\end{table}

\section{Theoretical Analysis}
\label{app:theory}

We analysis the stability of mean-space preference constraints to justify the design choice of imposing the preference objective on the squashed mean action $\bar{a}_\theta(s) = \tanh(\mu_\theta(s))$ rather than on log-probabilities, by showing that the log-probability formulation has unbounded worst-case gradients as policy variance decreases, whereas the mean-space formulation remains well-conditioned throughout training.

\paragraph{Setup.}
Let the policy be a $D$-dimensional tanh-squashed Gaussian, where the $d$-th pre-squash variable satisfies $u_d \sim \mathcal{N}(\mu_d, \sigma_d^2)$ with squashed action $a_d = \tanh(u_d)$.
We assume actions are clipped to $(-1+\eta, 1-\eta)$ for some small $\eta > 0$, ensuring $\tanh^{-1}(a_d)$ is well-defined.
The log-probability takes the form:
\begin{equation}
  \log\pi(a_d|s) = -\frac{(\tanh^{-1}(a_d) - \mu_d)^2}{2\sigma_d^2} - \log\sigma_d - \log(1-a_d^2) + \text{const},
\end{equation}
revealing an explicit $O(\sigma_d^{-2})$ dependence on the policy variance.
In entropy-regularized actor-critic training, the learned policy variance may become small as training converges, making log-probability-based preference losses potentially unstable, as formalized below.

\begin{proposition}[Worst-case instability of log-probability preference]
\label{prop:gradient}
Let $u_d^h = \tanh^{-1}(a_d^h)$, $u_d^\pi = \tanh^{-1}(a_d^\pi)$, and $\delta_d = u_d^h - u_d^\pi \neq 0$.
The gradient of $\mathcal{L}_\text{log} = -\log\sigma(\log\pi(a^h|s) - \log\pi(a^\pi|s))$ with respect to $\mu_d$ is not uniformly bounded as $\sigma_d \to 0$.
In particular, at $\mu_d = (u_d^h + u_d^\pi)/2$:
\begin{equation}
  \left|\frac{\partial \mathcal{L}_\text{log}}{\partial \mu_d}\right| = \frac{|\delta_d|}{2\sigma_d^2} \to \infty \quad \text{as } \sigma_d \to 0.
\end{equation}
By contrast, the mean-space preference loss
\begin{equation}
  \mathcal{L}_\text{pref} = \|\bar{a}_\theta - a^h\|^2 + \frac{1}{\|\bar{a}_\theta - a^\pi\|^2 + \epsilon},
\end{equation}
where $a^h$ and $a^\pi$ are fixed actions from the preference buffer with no gradient backpropagation, satisfies the per-dimension bound:
\begin{equation}
  \left|\frac{\partial \mathcal{L}_\text{pref}}{\partial \mu_d}\right| \leq 4 + \frac{4}{\epsilon^2} \quad \forall\, \mu_d \in \mathbb{R},\; \sigma_d > 0,
\end{equation}
and the vector gradient norm satisfies $\|\nabla_\mu \mathcal{L}_\text{pref}\|_2 \leq \sqrt{D}(4 + 4/\epsilon^2)$, independent of $\sigma_d$.
\end{proposition}

\begin{proof}
At $\mu_d = (u_d^h + u_d^\pi)/2$, the log-probability difference reduces to:
\begin{equation}
  \log\pi(a_d^h|s) - \log\pi(a_d^\pi|s)
  = \frac{\delta_d(2\mu_d - u_d^h - u_d^\pi)}{2\sigma_d^2} = 0.
\end{equation}
Therefore $1 - \sigma(0) = 1/2$, and the gradient is:
\begin{equation}
  \left|\frac{\partial \mathcal{L}_\text{log}}{\partial \mu_d}\right|
  = \frac{1}{2} \cdot \frac{|\delta_d|}{\sigma_d^2} \to \infty \quad \text{as } \sigma_d \to 0,
\end{equation}
establishing that the gradient is not uniformly bounded.

For the mean-space preference loss, the gradient with respect to $\mu_d$ is:
\begin{equation}
  \frac{\partial \mathcal{L}_\text{pref}}{\partial \mu_d}
  = \left(2(\tanh(\mu_d) - a_d^h)
  - \frac{2(\tanh(\mu_d) - a_d^\pi)}{(\|\bar{a}_\theta - a^\pi\|^2 + \epsilon)^2}\right)\mathrm{sech}^2(\mu_d).
\end{equation}
Applying the triangle inequality and bounding each factor using $|\tanh(\mu_d) - a_d^h| \leq 2$, $|\tanh(\mu_d) - a_d^\pi| \leq 2$, $\mathrm{sech}^2(\mu_d) \leq 1$, and $(\|\bar{a}_\theta - a^\pi\|^2 + \epsilon)^2 \geq \epsilon^2$:
\begin{equation}
  \left|\frac{\partial \mathcal{L}_\text{pref}}{\partial \mu_d}\right|
  \leq 4 + \frac{4}{\epsilon^2},
\end{equation}
uniformly over all $\mu_d \in \mathbb{R}$ and $\sigma_d > 0$.
Summing over $D$ dimensions and applying the Cauchy-Schwarz inequality yields $\|\nabla_\mu \mathcal{L}_\text{pref}\|_2 \leq \sqrt{D}(4 + 4/\epsilon^2)$.
\end{proof}

\section{Additional Experimental Results}

\subsection{Q-Value Correction Analysis}
\begin{figure}[!t]\centering
    \includegraphics[width=12cm]{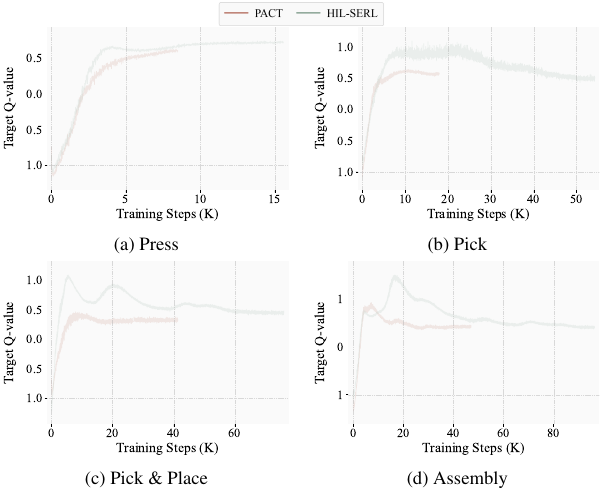}
    \caption{Target Q-value evolution during training across tasks. HIL-SERL produces inflated Q-value estimates that peak early and decay slowly, while PACT maintains consistently more conservative estimates throughout training.}
    \label{fig:target_q}
\end{figure} 

\begin{figure}[!h]\centering
    \includegraphics[width=11cm]{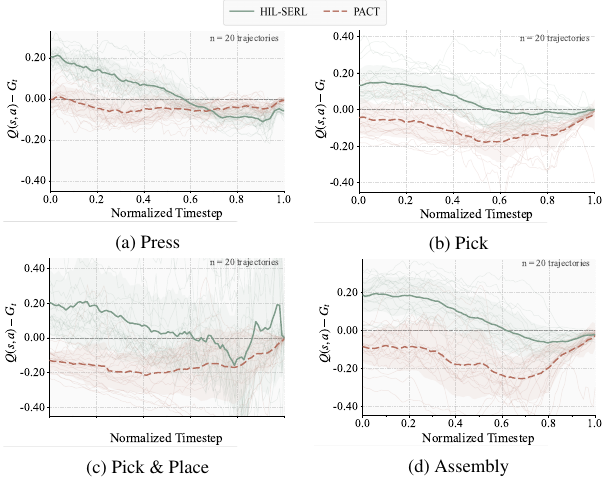}
    \caption{Q-value correction analysis across another four tasks. Each subfigure shows the normalized critic bias $Q(s,a) - G_t$ over intervention-containing successful episodes, comparing HIL-SERL and PACT.}
    \label{app:qvalue}
\end{figure} 

Figure~\ref{fig:target_q} extends the target Q-value evolution analysis of Section~4.4 to the remaining tasks.
In all cases, HIL-SERL produces rapidly increasing target values in the early stages of training, reflecting the propagation of terminal rewards through suboptimal pre-intervention segments via uniform TD backup.
In contrast, PACT consistently maintains more conservative Q-value estimates throughout training.
On more complex tasks, the gap is particularly pronounced, with HIL-SERL exhibiting sharp early peaks followed by slow decay, while PACT converges to a stable and lower estimate.
On simpler tasks, the two methods show closer trajectories, though PACT still maintains a consistently lower target Q-value.

To further verify that this suppression persists at convergence, Figure~\ref{app:qvalue} reports the normalized critic bias $Q(s,a) - G_t$ over intervention-containing successful episodes at the end of training, where $G_t = \sum_{k=0}^{T-t-1}\gamma^k r_{t+k}$ is the Monte Carlo return used as a diagnostic reference.
Across all tasks, HIL-SERL exhibits a consistently positive bias in the early and middle phases of each trajectory, while PACT suppresses this overestimation throughout, maintaining critic bias close to or below zero.
Together, the two analyses confirm that preference-aware segment-level correction effectively reduces Q-value inflation both during training and at convergence, generalizing the finding of Section~4.4 across tasks of varying difficulty.

\subsection{Case Study}
Figure~\ref{case_study} presents per-trajectory Q-value profiles for representative rollouts across all five tasks. Each case visualizes the Q-value estimates of HIL-SERL and PACT alongside the Monte Carlo return $G_t$, with the identified suboptimal segment and human intervention region highlighted.
The corresponding wrist-camera image sequences are shown below to illustrate the failure mode.
 
Across diverse failure patterns, a consistent pattern emerges: within the suboptimal segment, HIL-SERL assigns Q-values that substantially exceed the MC return.
In contrast, PACT suppresses the Q-value estimate within the identified segment, bringing it closer to or below the MC return. These single-trajectory cases complement the aggregate statistics of Section~4.4 by showing that PACT's credit correction operates at the level of individual suboptimal behaviors, correctly identifying and penalizing degraded execution regardless of failure type or task structure.
\begin{figure}[b]\centering
    \includegraphics[width=12cm]{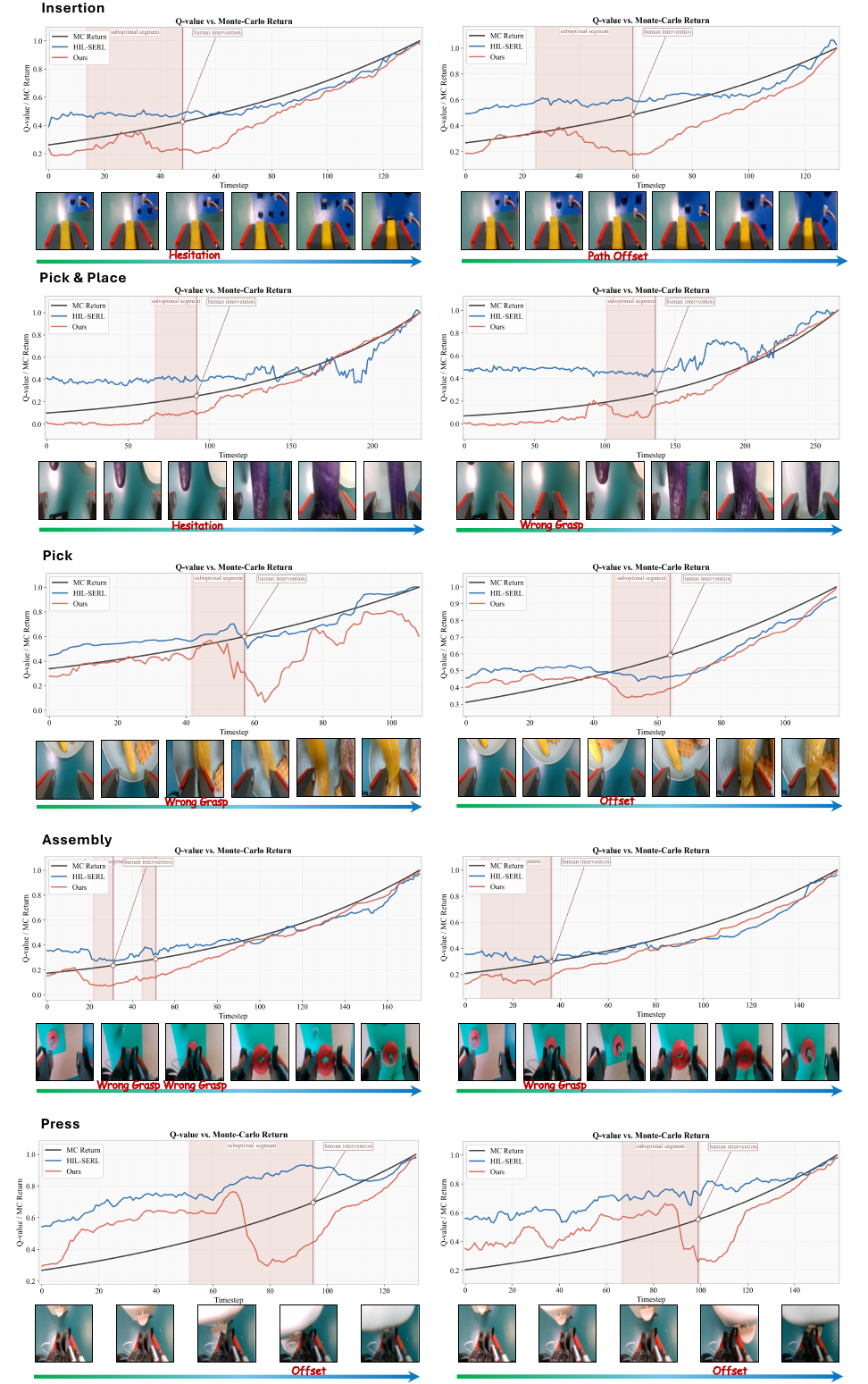}
    \caption{Per-trajectory Q-value profiles across all five tasks. Each case shows the Q-value estimates of HIL-SERL and PACT alongside the Monte Carlo return $G_t$, with the suboptimal segment (pink) and human intervention region (gray) highlighted. Wrist-camera image sequences below each curve illustrate the corresponding failure mode. Within the suboptimal segment, HIL-SERL consistently overestimates Q-values relative to $G_t$, while PACT assigns more conservative estimates that are closer to the Monte Carlo diagnostic reference.}
    \label{case_study}
\end{figure} 

\section{Algorithm Workflow for PACT}
PACT begins with an offline phase that collects human demonstrations to train the progress model $f_\psi$ via self-supervised learning. During online training, the \textit{actor process} continuously interacts with the environment under HIL supervision: during online training stage, the actor process interacts with the environment under HIL supervision and stores online transitions into the replay buffer. After each intervention-containing episode, PACT runs $f_\psi$ over the full trajectory, localizes suboptimal segments $[t_a,t_r]$ for each intervention point $t_h$, and stores the corresponding preference pairs $(a^{\mathrm{human}}, a^\pi)$ into $\mathcal{D}_{\mathrm{pref}}$. Concurrently, the \textit{learner process} samples from the replay buffer and applies preference-calibrated updates to both the critic, via position-weighted Bellman target correction, and the actor via  $\mathcal{L}_{\mathrm{pref}}$. Algorithm~\ref{alg:pact} summarizes the complete training pipeline.

\begin{algorithm}[h]
\caption{PACT: Preference-Calibrated Actor-Critic Training}
\label{alg:pact}
\begin{algorithmic}[1]
\REQUIRE Policy $\pi_\theta$, critic $Q_\phi$, replay buffer $\mathcal{D}_{replay}$, demo buffer $\mathcal{D}_{\mathrm{demo}}$
\STATE Collect human demonstration trajectories, store in $\mathcal{D}_{\mathrm{demo}}$ 
\STATE Train progress model $f_\psi$ on $\mathcal{D}_{\mathrm{demo}}$ via $\mathcal{L}_{\mathrm{progress}}$ 
\STATE \textbf{// Actor Process}
\WHILE{not converged}
    \STATE Collect one episode under policy execution or human intervention and store it in $\mathcal{D}_{\mathrm{replay}}$ 
    \IF{the episode contains intervention}
        \STATE Run $f_\psi$ over the full trajectory to localize $[t_a,t_r]$ for each intervention point $t_h$
        \STATE Mark $[t_a,t_r]$ in $\mathcal{D}_{\mathrm{replay}}$ and store $(a^{\mathrm{human}}, a^\pi, s_{t_h})$ in $\mathcal{D}_{\mathrm{pref}}$ 
    \ENDIF
\ENDWHILE
\STATE \textbf{// Learner Process}
\WHILE{not converged}
    \STATE Sample minibatch from $\mathcal{D}_{replay}$  and $\mathcal{D}_{\mathrm{demo}}$
    \STATE Compute counterfactual advantage $\mathcal{A}_{\mathrm{cf}}$ from $\mathcal{D}_{\mathrm{pref}}$ 
    \STATE Correct Bellman targets on $[t_a, t_r]$ with position-weighted $\mathcal{A}_{\mathrm{cf}}$ 
    \STATE Update critic $Q_\phi$ via corrected Bellman target $\mathcal{L}_{\mathrm{critic}}$ 
    \STATE Compute preference auxiliary loss $\mathcal{L}_{\mathrm{pref}}$ in bounded mean space 
    \STATE Update actor $\pi_\theta$ via $\mathcal{L}_{\mathrm{RLPD}} + \lambda_{\mathrm{pref}} \cdot \mathcal{L}_{\mathrm{pref}}$ 
    \STATE Soft-update target critic $\bar{\phi}$ 
\ENDWHILE
\end{algorithmic}
\end{algorithm}

\end{document}